\documentclass[12pt]{article}

\usepackage{graphicx}
\usepackage{subfigure}
\usepackage{amsmath} 
\usepackage{amssymb}  
\usepackage{url}
\usepackage{color}
\usepackage[english]{babel}
\usepackage{verbatim}
\usepackage{bbold}
\usepackage{mathtools}

\fussy
\newtheorem{theorem}{Theorem}

\newtheorem{definition}{Definition}
\newtheorem{property}{Property}
\newtheorem{remark}{Remark}

\newcommand{\qedstar}{\hfill $\star$}
\newcommand{\qed}{\hfill $\square$}
\newcommand{\calC}{ {\cal C}}

\newcommand{\prob}{\mathbb P}

\newcommand{\E}{\mathbb{E}}
\newcommand{\V}{\mathbb{V}\mathrm{ar}}
\newcommand{\Real}[1]{ { {\mathbb R}^{#1} } }
\newcommand{\dd}{\mathrm{d}}
\newcommand{\ee}{\mathrm{e}}

\newcommand{\One}[1]{ { {\mathrm 1}_{#1} } }
\newcommand{\eps}{\epsilon}

\newcommand{\bigO}[1]{{ \scalebox{1.2}{$O$} \left( {#1} \right) }}

\newcommand{\bigTh}[1]{{ \scalebox{1.2}{$\Theta$} \left( {#1} \right) }}
\newcommand{\littleo}[1]{{ \scalebox{1.1}{$o$} \left( {#1} \right) }}

\newcommand{\calH}{\mathcal{H}}

\newcommand{\scZ}{\mathcal{Z}}

\newcommand{\bz}{{\boldsymbol{z}}}

\newcommand{\bk}{{\boldsymbol{k}}}
\newcommand{\bki}{{\boldsymbol{k_i}}}
\newcommand{\bke}{{\boldsymbol{k_e}}}
\newcommand{\bphi}{{\boldsymbol{\phi}}}

\renewcommand{\Pr}{{\mathbb{P}}}
\newcommand{\loss}{\ell}
\newcommand{\algo}{\mathcal{A}}
\newcommand{\risk}{\mathsf{R}}
\newcommand{\ms}{\mathsf{ms}}
\newcommand{\co}{\mathsf{c}}

\title{\LARGE \bf
Learning by Compression: \\ Risk Estimates Attain \\ the Law-of-Large-Numbers Rate
}

\author{
	S.~Garatti\thanks{
		Dipartimento di Elettronica, Informazione e Bioingegneria - Politecnico di Milano, piazza Leonardo da Vinci 32, 20133 Milano, Italia. E-mail: {\sf simone.garatti@polimi.it}, web-site: {\sf https://garatti.faculty.polimi.it/} }
	\and
	M.C.~Campi\thanks{
		Dipartimento di Ingegneria dell'Informazione - Universit\`a di Brescia, via Branze 38, 25123 Brescia, Italia. E-mail: {\sf marco.campi@unibs.it}, web-site: {\sf https://marco-campi.unibs.it/} }
}

\begin{document}

\maketitle
\thispagestyle{empty}


\begin{abstract}
Compression schemes provide a general framework for deriving data-dependent risk bounds for learning algorithms. For learning rules based on preferent compression, existing results upper-bound the risk in terms of a sample-dependent complexity $\bk$, defined as the cardinality of the compression augmented by the training examples on which the reconstructed hypothesis is inappropriate. These results, however, yield a margin of order $O(\sqrt{\ln N/N})$ above $\bk/N$. In this paper, we establish that the logarithmic factor is not intrinsic. We prove that, for every fixed confidence level \(1-\delta\), the risk is bounded with probability at least \(1-\delta\) by
$$
\frac{\bk}{N} + O\left(\frac{1}{\sqrt{N}}\right),
$$
uniformly over all possible values of $\bk$. Thus, risk estimates for learning algorithms based on preferent compression attain the same law-of-large-numbers rate as the empirical estimate of the probability of a fixed event. We also show that the leading term $\bk/N$ is necessary and investigate its relation to the empirical risk. In particular, we identify regimes in which the empirical risk retains the $O(1/\sqrt{N})$ convergence rate, as well as regimes in which it converges more slowly or ceases to be consistent, while the complexity-based estimator preserves the $O(1/\sqrt{N})$ rate.
\end{abstract}


\section{Introduction} \label{sec:intro}

For a given event $A$, the empirical mean $\frac{1}{N}\sum_{i=1}^N  \mathbb{1} (\bz_i \! \in \! A)$ computed over an independent and identically distributed (i.i.d.) sequence $\bz_1,\bz_2,\ldots $ yields a consistent estimator of the event's probability $p(A)$.\footnote{Throughout, boldface indicates random quantities. For example, $\bz_i$ indicates a random element while a specific value it takes is indicated by $z_i$.} Moreover, adding a margin of $\bar{\varepsilon}/\sqrt{N}$ to the empirical mean, where $\bar{\varepsilon}$ satisfies relation $\sqrt{\frac{2}{\pi}} \int_{\bar{\varepsilon}}^{\infty} \ee^{-2x^2}dx = \delta$, provides an upper bound to $p(A)$ with an asymptotic confidence of $1 - \delta$: 
\begin{equation}
\label{margin empirical mean}
\prob \left\{ p(A) \leq \frac{1}{N}\sum_{i=1}^N  \mathbb{1} (\bz_i \! \in \! A) + \frac{\bar{\varepsilon}}{\sqrt{N}} \right\} \xrightarrow[N \to \infty]{} \alpha \geq 1 - \delta. 
\end{equation}
The rate $1/\sqrt{N}$ of the margin is not improvable, and indeed any margin that goes to zero faster than $1/\sqrt{N}$ is invalid for any choice of $\delta$ smaller than $1/2$ whenever $p(A)$ is not zero or one. This fact can be expressed as ``$1/\sqrt{N}$ \emph{is the correct rate for the margin}''. In applications, this result is relevant to estimate the \emph{risk} of falling in an undesirable event via a Monte-Carlo approach. For instance, the probability of failure in the transmission of strings through a communication channel can be estimated with asymptotic confidence $1 - \delta$ using the observed empirical mean of corrupted strings augmented by the margin $\bar{\varepsilon}/\sqrt{N}$. A justification of these facts, indeed standard in probability theory and statistical learning, can be found in Appendix \ref{appendix empirical mean}. 

Estimating the probability of a fixed event $A$, as discussed above, is a routine task. In Machine Learning (ML), data are used to design classifiers, predictors, and, more generally, hypotheses, and a crucial goal is to estimate the risk associated to these designs, for example the probability of misclassification of a classifier, \cite{DevGyoLug96}. This task of simultaneously designing and estimating the risk transcends, and is in fact far more complex than, estimating the probability of a fixed event $A$. This is because the event in which, for instance, a classifier fails depends on the classifier itself, which is a stochastic object constructed from the data. 

Significant advances in the problem of estimating the risk associated to a variety of ML algorithms have been achieved over the past few decades through the use of \emph{compression schemes}. While we will be more specific with precise definitions and technical results in the next Section \ref{section compression}, some fundamental concepts are briefly recalled here for the purposes of this introduction. According to a definition introduced in \cite{LittlestoneWarmuth1986}, given a learning algorithm $\algo$ that maps training sets $S$ into hypotheses $h = \algo(S)$ (for example, classifiers),  a compression scheme consists of (i) a \emph{compression function} $\co$, which maps any training set $S$ into a sub-set $\co(S)$, and (ii) a \emph{reconstruction function} $\rho$ that sends the \emph{compressed training set} $\co(S)$ into the hypothesis generated by $\algo$: $\rho(\co(S)) = \algo(S)$. An important feature of $\algo(S)$ is its \emph{risk}, which is defined as follows. Given a hypothesis $h$ (for example, a classifier or a predictor), one first defines a condition under which $h$ is deemed \emph{appropriate} for an \emph{example} $z$. For instance, a classifier $h$ is appropriate for an example $z = (x,y)$, where $x$ is an \emph{instance} and $y$ the corresponding \emph{label}, if $h(x) = y$; that is, if applying $h$ to $x$ correctly predicts the label $y$. With this premise, the risk of $\algo(S)$ is the probability that $\algo(S)$ is inappropriate for a new example $\bz$. 

In loose terms, most results in compression schemes relate the risk of $\algo(S)$ to the cardinality of the compression. In the data-consistent framework (i.e., the hypothesis is appropriate for all data points in the training set) and under the assumption that the maximum cardinality of $\co(S)$ is bounded by an integer $d$ (this is called the ``size'' of the compression scheme), \cite{LittlestoneWarmuth1986} and \cite{FloydWarmuth1995} establish results of the type: with high probability $1 \! - \! \delta$ with respect to the generation of the training set $S$, the risk of $\algo(S)$ is below a known bound that depends only on the size of $S$, $d$ and $\delta$. Extensions to a non-consistent framework are provided in \cite{GraepelHerbrichShawe-Taylor2005}. In a series of recent papers, compression schemes have been further investigated under a condition of \emph{stability}, a concept natural in many learning contexts and that has its roots in \cite{VapnikChervonenkis1974}, yielding tighter bounds than those obtained for general algorithms both in the consistent, \cite{pmlr-v125-bousquet20a}, and the non-consistent, \cite{HannekeKontorovich_2021}, settings. 

The primary reference for this manuscript is \cite{CampiGaratti2023}, which provides the tightest results for bounding the risk under stability.\footnote{In \cite{CampiGaratti2023}, the notion of stability is called \emph{preference}.} Specifically, in a general, non-consistent framework and without any upper bound on the cardinality of $\co(S)$, \cite{CampiGaratti2023} proves that the risk can be upper bounded with probability $1 \! - \! \delta$ by an explicitly defined function that takes as its argument the cardinality of $\co(S)$ plus the number of remaining examples for which the hypothesis $\algo(S)$ is inappropriate.\footnote{For instance, if the compressed training set has cardinality of $8$, and the hypothesis is inappropriate for $3$ additional examples (excluding the $8$ examples in the compressed set to avoid double counting in case some of them correspond to inappropriateness), then the function's argument is $8+3 = 11$.} This sum is a \emph{statistic} of the data, which will be denoted by symbol $\bk$ in what follows, and referred to as the \emph{complexity}.  

Building on the results of \cite{CampiGaratti2023}, this manuscript establishes the \emph{convergence rate} of the risk estimates in any learning algorithm based on stable compression. Specifically, letting $N$ be the cardinality of $S$, we show that any valid upper bound to the risk expressed as a function of $\bk$ must contain the term $\bk/N$ plus a margin, and prove that this margin tends to zero as $O(1/\sqrt{N})$, uniformly in the value taken by the statistic $\bk$.\footnote{For a positive function $g(N)$, notation $O(g(N))$ indicates a positive function $f(N)$ such that $\limsup_{N \to \infty} f(N)/g(N) < \infty$.} When estimating the probability $p(A)$ of a fixed event $A$, this result reduces to the known convergence rate of the empirical mean estimator in \eqref{margin empirical mean}. This manuscript proves that \emph{the same convergence rate extends to all learning frameworks based on stable compression}, a fact that may come as a significant surprise. In simple terms, this means that using the data to learn does not slow down the rate at which the risk can be estimated.

The paper is organized as follows. Section \ref{section compression} presents relevant facts about stable compression schemes and their use in evaluating the risk in ML applications. The main result on the convergence rate is given in Section \ref{sec:main-result}, where we also provide a detailed discussion of how our findings relate to the empirical risk. The proof of the main result is postponed to Section \ref{section-proof}. 

\section{Preferent compression schemes and risk evaluation} 
\label{section compression}

In this section, we revisit some findings from \cite{CampiGaratti2023}, which set the stage for presenting the results of this manuscript. We begin by introducing the notation, followed by a presentation of the relevant results. 

\subsection{Notation}

Examples $z$ are elements from a set $\scZ$ (for instance, in classification $z$ are pairs $(x,y)$ where $x$ is an \emph{istance} and $y$ is the corresponding \emph{label}). 

All compression functions are permutation invariant. Correspondingly, given a list of examples $z_1,\dots,z_N$, we introduce the associated multiset written as $\ms(z_1,\dots,z_N)$, where the operator ``$\ms$'' removes the ordering in the list, while maintaining repetitions. The set operations $\cup,\cap,\setminus$ (union, intersection, and set difference) are easily extended to multisets using the notion of multiplicity function $\mu_U$ for a multiset $U$, which counts how many times each element of $\scZ$ occurs in $U$. Then, $\mu_{U \cup U'}(z) = \mu_U(z)+\mu_{U'}(z)$, $\mu_{U \cap U'}(z) = \min\big\{\mu_U(z),\mu_{U'}(z)\big\}$, and $\mu_{U \setminus U'}(z) = \max\big\{0,\mu_U(z)-\mu_{U'}(z)\big\}$. Moreover, $U \subseteq U'$ means that $\mu_U(z) \leq \mu_{U'}(z)$ for all $z$, and $|U|$ stands for the cardinality of a multiset where each example is counted as many times as is its multiplicity. A compression function $\co$ is a map from any finite multiset of examples $U$ of any size to a sub-multiset: $\co(U) \subseteq U$. We write $\co(z_1,\ldots,z_N)$ as a shorthand for $\co(\ms(z_1,\ldots,z_N))$. Also, given a multiset $U$ and one more example $z$, $\co(U,z)$ stands for $\co(U \cup \ms(z))$. Similar notations apply to other maps having a multiset as argument. 

Throughout, an example is modeled as a realization of a random element defined over a probability space $(\Omega,\mathcal{F},\Pr)$. A training set of size $N$ is a multiset of $N$ observed examples generated as realizations of $N$ independent and identically distributed (i.i.d.) random elements $\bz_1,\bz_2,\ldots,\bz_N$. In ML applications, a learning algorithm $\algo$ is a map from training sets to a hypothesis $h$ in a set $\calH$ (in supervised binary classification, $h$ is a concept; in supervised learning with continuous label, $h$ is a predictor; in unsupervised learning, $h$ can, e.g., be a collection of clusters; etc.). We write $\algo(z_1,\dots,z_N)$ to denote the hypothesis generated by $\algo$ when the input is the training set $\ms(z_1,\dots,z_N)$. We use a $\{0,1\}$-valued function $\loss(h,z)$ to indicate whether or not a hypothesis $h$ is \emph{appropriate} for an example $z$: $\loss(h,z) = 0$ signifies that $h$ is \emph{appropriate} for $z$, while $\loss(h,z) = 1$ corresponds to \emph{inappropriateness} (for instance, in supervised classification, $\loss(h,z) = \mathbb{1}(h(x) \neq y)$). The \emph{statistical risk}, or more simply the \emph{risk}, of $h$ is $\risk(h) = \Pr\{\loss(h,\bz) = 1\}$, where $\bz$ is a random element distributed as each $\bz_i$. 

\subsection{Change of compression and risk bounds}
\label{section-change of compression and risk bounds}

The probability of change of compression is defined in \cite{CampiGaratti2023} as follows.  
\begin{definition}[probability of change of compression] 
\label{probability-of-change}
$$
\bphi_N = \Pr \{ \co(\co(\bz_1,\ldots,\bz_N),\bz) \neq \co(\bz_1,\ldots,\bz_N) | \bz_1,\ldots,\bz_N \},
$$
where $\bz$ is a random element distributed as each $\bz_i$.\footnote{This means that $\bphi_N$ is any version of the conditional probability on the right-hand side. 
}
\qedstar
\end{definition}
The following is a standing assumption in the analysis of \cite{CampiGaratti2023}. 
\begin{property}[preference]
	\label{preference}
	For any multisets $U$ and $V$ such that $V \subseteq U$, if $V \neq \co(U)$, then $V \neq \co(U,z)$ for all $z \in \scZ$. \qedstar
\end{property}
As shown in \cite{CampiGaratti2023}, \emph{preference} is a property that is satisfied by many compression-based ML techniques. As a consequence, the resulting generalization theory is broadly applicable, and \cite{CampiGaratti2023} explicitly discusses Support Vector Machines (SVM), \cite{CortesVapnik1995,SchSmola_BOOK,CampiGaratti2021}, the Guaranteed Error Machine (GEM), \cite{Campi2010,CareRamponiCampi2018}, as well as a list of other methods including the Nearest-Neighbor (NN) algorithm, \cite{Shalev-Shwartz_Ben-David_2014,KonSabWei2017,KonSabUrn2018,HanKonSabWei2021}, and the Total Recall (TR) algorithm, \cite{HelSloWar_1990}. Moreover, building on an idea from \cite{CGR2015}, \cite{P2L2023} and \cite{P2L2026} introduced a meta-algorithm that transforms any learning algorithm into a compression scheme that satisfies preference, thereby uncovering the remarkably broad applicability of Property~\ref{preference}. 

 Under preference, the probability of change of compression can be upper bounded by a function of the \emph{compression's cardinality}, as stated in the following Theorem~\ref{th:compression_1}. While the reader can find in \cite{CampiGaratti2023} a complete description of the function $\eps_{(\cdot)}$ appearing in the theorem's statement, for the purposes of this paper this expression is inconsequential and, therefore, it is omitted in the interest of brevity.\footnote{The first part of this theorem is stated in \cite{CampiGaratti2023} as Theorem~1 while result \eqref{expression for eps_k} is given in Proposition 8 of the same paper.} 
\begin{theorem} 
\label{th:compression_1}
	Assume the \emph{preference} Property \ref{preference}. For any $\delta \in (0,1)$, it holds that
	\begin{equation} 
	\label{result-th-1}
		\Pr \{ \bphi_N \leq \eps_{|\co(\bz_1,\ldots,\bz_N)|} \} \geq 1-\delta,
	\end{equation}
	where $\eps_{|\co(\bz_1,\ldots,\bz_N)|}$ is the random variable obtained by the composition of $|\co(\bz_1,\ldots,\bz_N)|$ (the cardinality of the multiset $\co(\bz_1,\ldots,\bz_N)$) with $\eps_{(\cdot)}$, a function defined  over the integers $k \in \{0, 1, \ldots, N\}$ that satisfies relation
\begin{equation}
\label{expression for eps_k}
	\eps_k \; \leq \;  \frac{k}{N} + 2 \frac{\sqrt{k+1}}{N}\left( \sqrt{\ln(k+1)} + 4 \right) + 2 \frac{ \sqrt{k+1}\sqrt{\ln \frac{1}{\delta}} }{N}  + \frac{\ln \frac{1}{\delta}}{N}. 
\end{equation}
	\qedstar
\end{theorem}
In \cite{CampiGaratti2023}, function $\eps_{(\cdot)}$ is assigned through an implicit procedure, and the importance of \eqref{expression for eps_k} lies in that it explicitly reveals the dependence of the bound on $k$ and $N$. It turns out that the first term (which, after substitution of the cardinality of the compressed set, gives the ratio $|\co(z_1,\ldots,z_N)|/N$) serves as an estimator for the probability of change of compression at a confidence level $1 - \delta$ provided that a margin (the second and third terms in the right-hand side of \eqref{expression for eps_k}) is added. As immediately seen, this margin goes to zero in $N$ as $O(\sqrt{\ln(N)/N})$, uniformly in $k = |\co(z_1,\ldots,z_N)| \in [0,N]$.

In \cite{CampiGaratti2023}, Theorem \ref{th:compression_1} is ported to the evaluation of the risk according to the following theorem.\footnote{This result is presented in \cite{CampiGaratti2023} as Theorem 17.} 

\begin{theorem}
	\label{Theorem risk}
	Consider a learning algorithm $\algo$ and a compression function $\co$ that satisfies the \emph{preference} Property~\ref{preference}. Assume that there exists a reconstruction function $\rho$  for $(\algo,\co)$, that is, $\rho(\big(\co(S)\big) = \algo(S)$ for any multiset $S$. Then, for any $\delta \in (0,1)$, it holds that 
	\begin{equation} 
	\label{bound risk}
		\Pr \left\{ \risk\big(\algo(\bz_1,\ldots,\bz_N)\big) \leq \eps_\bk \right\} \geq 1-\delta,
	\end{equation}
	where $\eps_{(\cdot)}$ is the same as in Theorem~\ref{th:compression_1} and, using the shorthand notation $\bz_1^N = \ms(\bz_1, \ldots, \bz_N)$, $\bk$ is defined as $\bk = \big| \co(\bz_1^N) \cup \big( \ms(\bz_i \! \in \! \bz_1^N: \loss(\algo(\bz_1^N),\bz_i) = 1) \setminus \co(\bz_1^N) \big) \big|$ (i.e., $\bk$ is the cardinality of $\co(\bz_1^N)$ augmented with the examples that are not already in $\co(\bz_1^N)$ for which $\algo(\bz_1^N)$ is \emph{inappropriate}). \qedstar
\end{theorem}
In view of \eqref{expression for eps_k}, Theorem \ref{Theorem risk} shows that adding to $\bk/N$ a margin of $O(\sqrt{\ln(N)/N})$ yields an upper bound to the risk at a confidence level $1 \! - \! \delta$.\footnote{In what follows, to avoid verbosity, we shall use shortened expressions like: ``\emph{provides a $1 \! - \! \delta$ bound to the risk}'' or ``\emph{is a $1 \! - \! \delta$ estimator of the risk}''.} 

Toward the analysis of the tightness of this result, a first, fundamental fact is that the term $\bk/N$ is strictly necessary. In fact, any valid replacement of $\eps_\bk$ in \eqref{bound risk} must necessarily take the form 
\begin{equation}
\label{need of margin}
\Pr \left\{ \risk\big(\algo(\bz_1,\ldots,\bz_N)\big) \leq \bk/N + \mbox{ margin} \right\} \geq 1-\delta 
\end{equation}
when $\delta < 0.5$, in which $\bk/N$ appears increased by a suitable nonnegative margin. See Appendix  \ref{appendix term added} for a proof of this result. The core issue that we address in this paper is then: \emph{what is the structure of this ``margin''?} 

An analysis of the result in \cite{CampiGaratti2023} shows that, using the function $\eps_{(\cdot)}$ there introduced, the margin of order $O(\sqrt{\ln(N)/N})$ given by \eqref{expression for eps_k} cannot be improved to $O(1/\sqrt{N})$. In other words, the presence of the term $\sqrt{\ln(N)}$ in the numerator is not due to any looseness in the evaluation \eqref{expression for eps_k}, it is intrinsic in the structure of $\eps_k - k/N$ of the function $\eps_{(\cdot)}$ in \cite{CampiGaratti2023}. While this fact could be proven rigorously, we omit such a proof here to avoid unnecessarily length. The consequence of this observation is that the results in \cite{CampiGaratti2023} leave the following question unanswered: can $\risk\big(\algo(\bz_1,\ldots,\bz_N)\big)$ be bounded at confidence level $1-\delta$ using $\bk/N$ with a margin of $O(1/\sqrt{N})$? While this conclusion cannot be drawn from the result in \cite{CampiGaratti2023}, it might nevertheless be attainable through a new theoretical approach that extends the analysis of that paper. This topic is the subject of the present contribution.

Turning back to the simple problem of establishing an upper bound for the probability of a fixed event $A$, a positive answer to the above question would simply re-establish the result presented at the beginning of this introduction. To see this, one can take as $\algo$ the algorithm that always (that is, independently of what data have been collected) returns the hypothesis $A^c$ (the complement of $A$), as $\co$ the compression function that always returns the empty multiset, while the reconstruction function $\rho$ sends the empty multiset to $A^c$. $A^c$ is inappropriate for $z$ if $z \notin A^c$, that is, $z \in A$. Then, $\risk\big(\algo(\bz_1,\ldots,\bz_N)\big) = p(A)$. In this context, $|\co(\bz_1,\ldots,\bz_N)| =0$, so that $\bk = \sum_{i=1}^N  \mathbb{1} (\bz_i \! \in \! A)$; hence, a positive answer to the above question would recover the fact that $\frac{1}{N}\sum_{i=1}^N  \mathbb{1} (\bz_i \! \in \! A) + O(1/\sqrt{N})$ is a valid upper bound for $p(A)$ with confidence $1-\delta$, as already established at the beginning of the introduction. Moving toward more general compression-based ML techniques, such a result would demonstrate that $\bk/N$ serves as a universal estimator of the risk with again a margin of $O(1/\sqrt{N})$, which is perhaps surprising given that this takes place in a complex setup involving the construction of the hypothesis from data. 

The next section indeed provides a positive answer: \emph{$\risk\big(\algo(\bz_1,\ldots,\bz_N)\big)$ can be bounded at confidence level $1-\delta$ using $\bk/N$ with a margin of $O(1/\sqrt{N})$.} 

\section{New convergence results} 
\label{sec:main-result}

The new results presented in this section as Theorems \ref{change of compression rate} and \ref{theorem risk bounded 1/sqrtN} establish the convergence rate for the estimators discussed in the previous section. 
\begin{theorem}
\label{change of compression rate}
Assume the \emph{preference} Property \ref{preference}. For any $\delta \in (0,1)$, it holds that 
\begin{equation} 
	\Pr \left\{ \bphi_N \leq \frac{|\co(\bz_1,\ldots,\bz_N)|}{N} + \mu(N,\delta) \right\} \geq 1-\delta, \nonumber 
\end{equation}
where, for any given $\delta \in (0,1)$, $\mu(N,\delta) = O(1/\sqrt{N})$. 
\qedstar
\end{theorem}
{\bf Proof.} See Section \ref{section-proof}. 
\qed \\

In the context of risk evaluation, Theorem \ref{change of compression rate} immediately yields the following new result. 

\begin{theorem}
\label{theorem risk bounded 1/sqrtN}
	Consider a learning algorithm $\algo$ and a compression function $\co$ that satisfies the \emph{preference} Property~\ref{preference}. Assume that there exists a reconstruction function $\rho$  for $(\algo,\co)$. Then, for any $\delta \in (0,1)$, it holds that 
	\begin{equation} 
		\label{asymptotic bound risk}
		\Pr \left\{ \risk\big(\algo(\bz_1,\ldots,\bz_N)\big) \leq \frac{\bk}{N} +  \mu(N,\delta) \right\} \geq 1-\delta,
	\end{equation}
	where $\bk$ is defined as in Theorem \ref{Theorem risk} and, for any given $\delta \in (0,1)$, $\mu(N,\delta) = O(1/\sqrt{N})$. \qedstar
\end{theorem}
{\bf Proof.} The proof follows that of Theorem~17 of \cite{CampiGaratti2023}, with the following changes. As in Theorem~17, one observes that $\co(\bz_1^N) \cup \big( \ms(\bz_i \! \in \! \bz_1^N: \loss(\algo(\bz_1^N),\bz_i) = 1) \setminus \co(\bz_1^N) \big)$ defines an extended compression function $\tilde{\co}(\bz_1^N)$ that is still preferent and whose probability of change of compression upper bounds the risk:
\begin{equation} \label{eq:R<=Prchange}
\risk\big(\algo(\bz_1,\ldots,\bz_N)\big) \leq \mbox{``probability of change of compression for } \tilde{\co}\mbox{''}.
\end{equation}
Then, resorting to Theorem \ref{change of compression rate} stated above, and recalling the definition of $\bk$, one finds that the ``probability of change of compression for $\tilde{\co}$'' is upper bounded by  $|\tilde{\co}(\bz_1,\ldots,\bz_N)|/N + \mu(N,\delta) = \bk/N + \mu(N,\delta)$ with confidence no smaller than $1-\delta$. This immediately gives \eqref{asymptotic bound risk} in view of \eqref{eq:R<=Prchange}. \qed

\subsection{Form empirical risk to complexity: a natural extension} 
\label{sec:implications}

In this section, we discuss interesting connections between the statistic $\bk/N$ and the empirical risk, which shed further light on the mutual relation between these two statistics. 

Let $\bki$ = number of examples in $\ms\{\bz_1, \ldots, \bz_N\}$ for which $\algo(\bz_1, \ldots, \bz_N)$ is inappropriate, and define $\bke := \bk - \bki$. $\bki/N$ is the \emph{empirical risk}, which we will at times write more expressively as $\hat{\risk}\big(\algo(\bz_1,\ldots,\bz_N)\big)$. Instead, $\bke/N$ is referred to as the \emph{excess}; it represents the increment that must be added to the empirical risk to obtain the statistic $\bk/N$. Using this notation, equation \eqref{asymptotic bound risk} can be rewritten as 
\begin{equation}
\label{R < hat{R} + margin}
\Pr \left\{ \risk\big(\algo(\bz_1,\ldots,\bz_N)\big) \leq \hat{\risk}\big(\algo(\bz_1,\ldots,\bz_N)\big) + \frac{\bke}{N} + \mu(N,\delta) \right\} \geq 1-\delta,
\end{equation}
which allows us to highlight key connections to using the empirical risk as an estimator.  

\noindent 
\textbf{Case 1: $\bke$ bounded uniformly in $N$}

In this case, $\frac{\bke}{N} + \mu(N,\delta) = O(1/\sqrt{N})$, and we see that the empirical risk is a $1 \! - \! \delta$ estimator of the risk with a margin of $O(1/\sqrt{N})$. This result is already known and can e.g. be found in \cite{HannekeKontorovich_2021}. Examples of this first case include SVR in finite dimensions without infinite dimensional lifting, the original GEM algorithm introduced in \cite{Campi2010}, and any learning scheme based on finite dimensional convex optimization. \qedstar 

For the case where $\bke$ grows unbounded with $N$, we distinguish various situations depending on the growth rate of $\E[\bke]$. Before proceeding, we show the validity of the following simple consequence of \eqref{R < hat{R} + margin} that we will use in later derivations: 
\begin{equation}
\label{R < hat{R} + margin expected}
\Pr \left\{ \risk\big(\algo(\bz_1,\ldots,\bz_N)\big) \leq \hat{\risk}\big(\algo(\bz_1,\ldots,\bz_N)\big) + \frac{\E[\bke]}{N}\frac{2}{\delta} + \mu(N,\delta/2) \right\} \geq 1-\delta. 
\end{equation}
{\bf Proof of \eqref{R < hat{R} + margin expected}.}
Omitting the arguments of $\risk$ and $\hat{\risk}$, we have 
\begin{eqnarray*}
\lefteqn{ \Pr \left\{ \risk \leq \hat{\risk} + \frac{\E[\bke]}{N}\frac{2}{\delta} + \mu(N,\delta/2) \right\} } \\
& \geq & \Pr \left\{ \risk \leq \hat{\risk} + \frac{\bke}{N} + \mu(N,\delta/2) \text{ and } \frac{\bke}{N} \leq \frac{\E[\bke]}{N}\frac{2}{\delta} \right\} \\
& \geq & 1 - \frac{\delta}{2} - \frac{\delta}{2} \quad \quad \mbox{(using \eqref{R < hat{R} + margin} and Markov's inequality)}\\
& = & 1-\delta.  	
\end{eqnarray*}
\qed

\noindent 
\textbf{Case 2: $\E[\bke] = O(\sqrt{N})$}

This case applies for instance to the construction of the convex hull in $\Real{2}$ of points randomly generated according to various distributions, as explained in what follows. Given a training set of points $z_1, \ldots, z_N$ in $\Real{2}$, consider the algorithm $\algo$ that returns their convex hull. The convex hull is \emph{appropriate} for a new $z$ if $z \in \mbox{convex hull}$. The compression function $\co$ singles out the vertexes of the convex hull and the reconstruction function $\rho$ coincides with algorithm $\algo$. In this case, $\bk_i = {\bf 0}$, while $\bk_e$ depends on the distribution by which data points are generated. In the seminal article \cite{RenyiSulanke63}, Renyi and Sulanke show that, for radially symmetric Gaussian distributions, one has $\E[\bke] = \Theta(\sqrt{\ln(N)})$, while for a uniform distribution over a disk centered at the origin, $\E[\bke] = \Theta(N^{\frac{1}{3}})$.\footnote{For a positive function $g(N)$, notation $\Theta(g(N))$ indicates a positive function $f(N)$ such that $\limsup_{N \to \infty} f(N)/g(N) < \infty$ and $\liminf_{N \to \infty} f(N)/g(N) > 0$.} Moreover, $\E[\bke]$ exactly hits the growth rate $\sqrt{N}$ (i.e., $\E[\bke] = \Theta(\sqrt{N})$) when the convex hull is constructed in $\Real{3}$ from points uniformly drawn from a ball centered at the origin (this follows from \eqref{growth convex hull in Rd}, presented in Case 3).\footnote{While in this example we have $\bk_i = {\bf 0}$, and thereby $\hat{\risk}\big(\algo(\bz_1,\ldots,\bz_N)\big) = {\bf 0}$, which may sound specific, a simple modification gives $\bk_i \neq {\bf 0}$ while preserving all conclusions about $\E[\bke]$. To this purpose, it is enough to select half-space and consider as output of $\algo$ the convex hull of the points that belong to this half-space.\label{footnoe half-space}}

In this Case 2, $\frac{\E[\bke]}{N}\frac{2}{\delta} + \mu(N,\delta/2)$ is still of order $O(1/\sqrt{N})$ and, using result \eqref{R < hat{R} + margin expected}, we see that also in this case the empirical risk $\hat{\risk}\big(\algo(\bz_1,\ldots,\bz_N)\big)$ serves as a $1 \! - \! \delta$ estimator of the risk with margin $O(1/\sqrt{N})$. \qedstar

\textbf{Case 3:} $\E[\bke] = \littleo{N}$

In this case, using \eqref{R < hat{R} + margin expected} no longer allows us to conclude that $\hat{\risk}\big(\algo(\bz_1,\ldots,\bz_N)\big)$ serves as a $1 \! - \! \delta$ estimator of the risk with margin $O(1/\sqrt{N})$. Indeed, if $\E[\bke]$ grows faster than $\sqrt{N}$, then the dominant term of $\frac{\E[\bke]}{N}\frac{2}{\delta} + \mu(N,\delta/2)$ becomes $\frac{\E[\bke]}{N}\frac{2}{\delta}$, which tends to zero at a rate slower than $1/\sqrt{N}$. It is indeed inherent that, when $\E[\bke] = \littleo{N}$, the empirical risk is in general no longer a suitable estimator of the risk $\risk\big(\algo(\bz_1,\ldots,\bz_N)\big)$ with a margin $O(1/\sqrt{N})$, as the following example shows. 

Consider the convex hull in $\Real{d}$ of points drawn uniformly from the unit ball. This case yields 
\begin{equation}
\label{growth convex hull in Rd}
\E[\bke] = \bigTh{N^{\frac{d-1}{d+1}}} \text{ and } \V[\bke] = \bigO{N^{\frac{d-1}{d+1}}} 
\end{equation}
(see \cite{Raynaud1970,Barany1992} for the first result and \cite{Reitzner2003} for the second, while \cite{Barany2008} provides an overview). Taking $d=4$, we obtain $\E[\bke] = \bigTh{N^{\frac{3}{5}}}$, and $\V[\bke] = \bigO{N^{\frac{3}{5}}}$. 

We show that, in this case, it is intrinsic that $\hat{\risk}\big(\algo(\bz_1,\ldots,\bz_N)\big)$ does not provide an upper bound to $\risk\big(\algo(\bz_1,\ldots,\bz_N)\big)$ with a margin of $O(1/\sqrt{N})$. Start by noting that Theorem~19 in \cite{CampiGaratti2023} establishes lower bounds to the risk under the additional assumptions of \emph{non-associativity}, \emph{non-concentrated mass}, and \emph{coherence - part I and II}. All these assumptions are easily seen to be satisfied in the present context, where $\co$ is the compression function that singles out the vertexes of the convex hull. Then, applying Theorem 19 and using the expression (8) in \cite{CampiGaratti2023} for the lower bound to the risk yields\footnote{Although in the current example $\hat{\risk}\big(\algo(\bz_1,\ldots,\bz_N)\big) = 0$ by construction, we prefer to retain this term to allow for evaluations that can be easily extended to variants, for example when the convex hull of the points belonging to a given half-space is considered, as suggested in Footnote \ref{footnoe half-space}.} 
\begin{equation}
\label{R(A)>hatR(A)}
\Pr \left\{ \risk\big(\algo(\bz_1,\ldots,\bz_N)\big) \geq \hat{\risk}\big(\algo(\bz_1,\ldots,\bz_N)\big) + \frac{\bke}{N} - \nu(N,\delta/2) \right\} \geq 1 - \frac{\delta}{2}
\end{equation}
where, for any given $\delta$, $\nu(N,\delta) = O\left(\sqrt{\ln(N)/N}\right)$. Moreover, to draw the result we need the following simple inequality: 
\begin{eqnarray}
\Pr \left\{ \frac{\bke}{N} \geq  \frac{\E[\bke]}{N} - \frac{1}{N}\sqrt{\frac{2 \V[\bke]}{\delta}} \right\} & \geq & \Pr \left\{ \left| \bke -  \E[\bke] \right| \leq  \sqrt{\frac{2 \V[\bke]}{\delta}} \right\} \nonumber \\
& \geq & 1 - \frac{\V[\bke]}{\left( \sqrt{\frac{2 \V[\bke]}{\delta}} \right)^2} \quad \quad \mbox{(using Chebyshev's inequality)} \nonumber \\
& = & 1 - \frac{\delta}{2}. \label{ke>Eke}
\end{eqnarray}
Thus, 
	\begin{eqnarray}
		\lefteqn{ \Pr \left\{ \risk\big(\algo(\bz_1,\ldots,\bz_N)\big) \geq \hat{\risk}\big(\algo(\bz_1,\ldots,\bz_N)\big) + \frac{\E[\bke]}{N} - \frac{1}{N}\sqrt{\frac{2 \V[\bke]}{\delta}}  - \nu(N,\delta/2) \right\} } \label{R>hatR+} \\
		& \geq & \Pr \left\{ \risk\big(\algo(\bz_1,\ldots,\bz_N)\big) \geq \hat{\risk}\big(\algo(\bz_1,\ldots,\bz_N)\big)+ \frac{\bke}{N} - \nu(N,\delta/2) \text{ and } \frac{\bke}{N} \geq \frac{\E[\bke]}{N} - \frac{1}{N}\sqrt{\frac{2 \V[\bke]}{\delta}} \right\} \nonumber \\
		& \geq & 1 - \frac{\delta}{2} - \frac{\delta}{2} \nonumber \\
		& = & 1-\delta, \nonumber
	\end{eqnarray}
where the penultimate step holds in view of \eqref{R(A)>hatR(A)} and \eqref{ke>Eke}. Considering that, of the last three terms in the left-hand side of \eqref{R>hatR+}, the dominant one is $\E[\bke] / N = \bigTh{1 / N^{\frac{2}{5}}}$, while the other two are  $(\sqrt{2\V[\bke]/\delta})/N = O(1 / N^{\frac{7}{10}})$ and $\nu(N,\delta) = O(\sqrt{\ln(N)/N})$, 
we obtain: 
$$
\Pr \left\{ \risk\big(\algo(\bz_1,\ldots,\bz_N)\big) \geq \hat{\risk}\big(\algo(\bz_1,\ldots,\bz_N)\big) + \bigTh{1/N^{\frac{2}{5}}} \right\} \geq 1 - \delta. 
$$
Therefore, any valid margin added to the empirical risk must converge slower than $O(1/\sqrt{N})$ in this case.  \qedstar

\textbf{Case 4:} $\E[\bke] = O(N)$

In this case, considerations similar to those in the previous case apply, yet the situation becomes even more dramatic. At this point the dominant term $\frac{\E[\bke]}{N}\frac{2}{\delta}$ in equation \eqref{R < hat{R} + margin expected}, does not even converge to zero. This is a symptom of the fact that the empirical risk, in general, ceases to be a consistent estimator. An example where this happens is the following. Consider again the convex hull of points in $\Real{2}$, this time drawn from a distribution that is uniform on a circumference. In this case, the risk $\risk\big(\algo(\bz_1,\ldots,\bz_N)\big)$ always equals $1$, since the whole circumference, except just for the drawn points, lies outside the convex hull. On the other hand, the empirical risk $\hat{\risk}\big(\algo(\bz_1,\ldots,\bz_N)\big)$ is always zero.~\qedstar

In Cases 3 and 4, the empirical risk no longer serves as a fast-converging estimator of the risk, and may even lose consistency. On the other hand, as we have seen, the more refined statistic $\bk/N$ (the \emph{complexity} divided by $N$) serves as an estimator of the risk that preserves fast convergence in a data-driven setting based on preferent compression schemes.

\section{Proof of Theorem \ref{change of compression rate}} 
\label{section-proof}

Toward the goal of proving Theorem \ref{change of compression rate}, we need first to revisit Theorem~1 in \cite{CampiGaratti2023}. A close inspection of the proof of this latter theorem shows that it actually establishes a result slightly stronger than what explicitly formulated as the theorem statement. This stronger result will serve here as the starting point for the proof of our Theorem \ref{change of compression rate}. 

The stronger result in the proof of Theorem~1 in \cite{CampiGaratti2023} is contained in equations (46)–(47) of \cite{CampiGaratti2023}. For convenience, we restate it here as Theorem~\ref{thm:main_compression}.

\begin{theorem} \label{thm:main_compression}
	Assume the \emph{preference} Property \ref{preference}. For any $[0,1]$-valued function $\eps_{(\cdot)}$ defined over the integers $k \in \{0,1,\ldots,N\}$, it holds that
	$$
		\Pr \{ \bphi_N \leq \eps_{|\co(\bz_1,\ldots,\bz_N)|} \} \geq 1-\gamma^\ast_N,
	$$
	where $\eps_{|\co(\bz_1,\ldots,\bz_N)|}$ is the random variable obtained by the composition of $|\co(\bz_1,\ldots,\bz_N)|$ (the cardinality of the multiset $\co(\bz_1,\ldots,\bz_N)$) 
	with $\eps_{(\cdot)}$, and 
	\begin{eqnarray} 
		\gamma^\ast_N = \inf_{\lambda_m \geq 0, m=0,\ldots,N} & & \sum_{m=0}^N \lambda_m \nonumber \\
		\textrm{\rm subject to:} & &  {N \choose k} (1-\alpha)^{N-k} \cdot \One{\alpha \in (\eps_k,1]} \leq  \sum_{m=k}^N \lambda_m {m \choose k} (1-\alpha)^{m-k}, \nonumber \\
		& & \forall \alpha \in [0,1], \quad \forall k=0,1,\ldots,N \label{aux_opt_pb}
	\end{eqnarray}
	($\One{A}$ is the indicator function of set $A$). \qedstar
\end{theorem}

In this result, the perspective is reversed compared with Theorem 1 in \cite{CampiGaratti2023}: rather than assigning a function $\eps_{(\cdot)}$ satisfying prescribed properties, Theorem \ref{thm:main_compression} takes as input an arbitrary user-chosen function $\eps_{(\cdot)}$ and returns the confidence with which $\eps_{|\co(\bz_1,\ldots,\bz_N)|}$ is a valid upper bound on the probability of change of compression $\bphi_N$. 

On the other hand, Theorem \ref{thm:main_compression} can also be used in the reverse direction: it allows one to compute a valid complexity-based upper-bound on the risk that holds with a user-specified confidence $1-\delta$. This is indeed the use that we make below of Theorem \ref{thm:main_compression}. Crucially, starting from Theorem~\ref{thm:main_compression}, rather than Theorem~1 in \cite{CampiGaratti2023}, provides the extra flexibility needed to establish the convergence rate stated in Theorem~\ref{change of compression rate}. 

To pursue this use of Theorem \ref{thm:main_compression}, we start from a selection of $\{\lambda_m\}_{m=0}^N$, i.e., a tentative solution to \eqref{aux_opt_pb}, such that $\lambda_m \geq 0$, $\forall m$, and $\sum_{m=0}^N \lambda_m = \delta$, and determines $\eps_k$, $k=0,\ldots,N$, so that the inequality in \eqref{aux_opt_pb} is satisfied. 
This ensures the feasibility of the chosen $\{ \lambda_m \}_{m = 0}^N$ for \eqref{aux_opt_pb}, and therefore the optimal solution to \eqref{aux_opt_pb} with the so-computed $\eps_{(\cdot)}$ must result in $\gamma^\ast \leq \sum_{m=0}^N \lambda_m = \delta$. An application of Theorem \ref{thm:main_compression} then yields $\Pr \{ \bphi_N \leq \eps_{|\co(\bz_1,\ldots,\bz_N)|} \} \geq 1-\delta$.

An inspection of the proof of Theorem~1 in \cite{CampiGaratti2023} shows that the expression for $\eps_{(\cdot)}$ given therein is in fact obtained by applying the procedure described above with $\lambda_m = \delta/N$ for $m=0,\ldots,N-1$ and $\lambda_N = 0$. This choice was suggested in \cite{CampiGaratti2023} because it yields an $\eps_{(\cdot)}$ that is tight and informative across the whole range of values for $k$, even for moderate values of $N$. However, as observed in Section \ref{section-change of compression and risk bounds}, the resulting $\eps_{(\cdot)}$ takes a margin from $k/N$ that vanishes asymptotically at rate $O(\sqrt{\ln(N)/N})$. To prove Theorem \ref{change of compression rate}, we instead consider the $\eps_{(\cdot)}$ function obtained from an alternative choice of the $\lambda_m$'s. While this choice is typically less favorable than that in Theorem~1 in \cite{CampiGaratti2023} for moderate values of $N$, as $N \to \infty$ it outperforms the choice made in \cite{CampiGaratti2023} yielding the sought result on the  convergence rate.

Let 
$$
\lambda_m =
\begin{cases}
	\delta, & \quad m =\bar{m} \\
	0 & \quad m \neq \bar{m}
\end{cases},
$$
where $\bar{m} = N - \lfloor \sqrt{N} \rfloor$ ($\lfloor \cdot \rfloor$ denotes integer part). 

With this choice of the $\lambda_m$'s, the inequality in \eqref{aux_opt_pb} becomes
\begin{subequations}
\begin{align}
	{N \choose k} (1-\alpha)^{N-k} \cdot \One{\alpha \in (\eps_k,1]} & \leq  \delta {\bar{m} \choose k} (1-\alpha)^{\bar{m}-k}, & \quad k = 0,\ldots,\bar{m}, \label{eq:ineq<=barm} \\
	{N \choose k} (1-\alpha)^{N-k} \cdot \One{\alpha \in (\eps_k,1]} & \leq  0, & \quad k = \bar{m}+1,\ldots,N. \label{eq:ineq>barm}
\end{align}
\end{subequations}
Since ${N \choose k} (1-\alpha)^{N-k} > 0$ for every $\alpha \in [0,1)$, condition \eqref{eq:ineq>barm} can hold only if $\eps_k = 1$ for $k = \bar{m}+1,\ldots,N$. Consider now \eqref{eq:ineq<=barm}. This inequality is automatically satisfied for $\alpha=1$. For $\alpha \in [0,1)$, division by the positive quantity ${N\choose k}(1-\alpha)^{N-k}$ yields
$$
\One{\alpha \in (\eps_k,1]} \leq \delta \frac{{\bar{m} \choose k}}{{N \choose k}} (1-\alpha)^{\bar{m}-N}.
$$
The right-hand side is smaller than $1$ for $\alpha = 0$, monotonically increasing, and diverging to $+\infty$ for $\alpha \to 1$. Hence, we can choose as $\eps_k$ the unique solution in $(0,1)$ to equation $1 = \delta \frac{{\bar{m} \choose k}}{{N \choose k}} (1-\alpha)^{\bar{m}-N}$, which is $1-\left( \delta \frac{{\bar{m} \choose k}}{{N \choose k}} \right)^{\frac{1}{N-\bar{m}}}$. 

Summarizing, with the choice 
\begin{equation} \label{eq:epsk4as_expression}
\eps_k = 
\begin{cases}
	1-\left( \delta \frac{{\bar{m} \choose k}}{{N \choose k}} \right)^{\frac{1}{N-\bar{m}}} & \quad k = 0,\ldots,\bar{m}, \\
	1 & \quad k = \bar{m}+1,\ldots,N,
\end{cases}
\end{equation}
Theorem~\ref{thm:main_compression} ensures that
\begin{equation} \label{eq:epsk_is_ok}
\Pr \{ \bphi_N \leq \eps_{|\co(\bz_1,\ldots,\bz_N)|} \} \geq 1-\delta, 
\end{equation}
where $\eps_k$ is given by \eqref{eq:epsk4as_expression}. 

Note now that, for $k = \bar{m}+1,\ldots,N$, the value $\eps_k = 1$ can be rewritten as 
\begin{align}
	\eps_k & = \; \frac{k}{N} + \left(1- \frac{k}{N} \right) \nonumber \\
	& \leq \; \frac{k}{N} + \left(1- \frac{\bar{m}}{N} \right) \nonumber \\
	& = \; \frac{k}{N} + \left(1- \frac{N-\lfloor \sqrt{N} \rfloor}{N} \right)
	\tag*{\big[recall that $\bar{m} = N-\lfloor \sqrt{N} \rfloor$\big]}\\
	& = \; \frac{k}{N} + \frac{\lfloor \sqrt{N} \rfloor}{N} \nonumber \\
	& \leq \; \frac{k}{N} + \frac{1}{\sqrt{N}}.
	\label{eq:bound_epsk_k>barm}
\end{align}
Thus, $\eps_k$ is bounded by $k/N$ plus the margin $1/\sqrt{N}$. 

To bound $\eps_k$ when $k = 0,\ldots,\bar{m}$, we first observe that
\begin{eqnarray}
\frac{{\bar{m} \choose k}}{{N \choose k}} =  \frac{\bar{m}!}{(\bar{m}-k)!} \frac{(N-k)!}{N!}  & = & \prod_{i=1}^{N-\bar{m}} \frac{\bar{m}-k+i}{\bar{m}+i} \nonumber \\
& \geq & \left( \frac{\bar{m}-k}{\bar{m}} \right)^{N-\bar{m}} = \left( 1-\frac{k}{\bar{m}} \right)^{N-\bar{m}}, \label{eq:bnd_bin_coeff}
\end{eqnarray}
where the inequality follows because $\frac{\bar{m}-k}{\bar{m}} \leq \frac{\bar{m}-k+i}{\bar{m}+i}$, for all $i\geq0$. Substituting \eqref{eq:bnd_bin_coeff} in the expression for $\eps_k$, $k = 0,\ldots,\bar{m}$, in \eqref{eq:epsk4as_expression} yields 
\begin{eqnarray*}
	\eps_k & \leq & 1 - \left( \delta \left( 1 -\frac{k}{\bar{m}} \right)^{N-\bar{m}} \right)^{\frac{1}{N-\bar{m}}} \\
	& = & 1 - \delta^{\frac{1}{N-\bar{m}}} \left( 1 -\frac{k}{\bar{m}} \right)  \\
	& = &  \frac{k}{N} +  \underbrace{\left(1 - \frac{k}{N}\right)- \delta^{\frac{1}{N-\bar{m}}} \left( 1 -\frac{k}{\bar{m}} \right)}_{\text{margin}}.
\end{eqnarray*}
The margin $\left(1 - \frac{k}{N}\right)- \delta^{\frac{1}{N-\bar{m}}} \left( 1 -\frac{k}{\bar{m}} \right)$ is affine in $k$. Therefore, its maximum over $k=0,\ldots,\bar{m}$ is attained at one of the endpoints, namely $k=\bar{m}$ or $k=0$. In the first case, i.e., $k=\bar{m}$, we have
$$
\text{margin} \leq 1-\frac{\bar{m}}{N} = 1- \frac{N-\lfloor \sqrt{N} \rfloor}{N} = \frac{\lfloor \sqrt{N} \rfloor}{N} \leq \frac{1}{\sqrt{N}}.
$$
For $k=0$, we instead have 
\begin{align*}
	\text{margin}  & = \; 1 - \delta^{\frac{1}{N-\bar{m}}}	\\
	& \leq \ln \left( \frac{1}{\delta^{\frac{1}{N-\bar{m}}} } \right) \tag*{\big[since $1-x \leq \ln(1/x)$\big]} \\
	& = \frac{\ln \left( \frac{1}{\delta} \right)}{N-\bar{m}} \\
	& = \frac{\ln \left( \frac{1}{\delta} \right)}{\lfloor \sqrt{N} \rfloor} \tag*{\big[recall that $\bar{m} = N-\lfloor \sqrt{N} \rfloor$ \big]} \\
	& \leq 2\frac{\ln \left( \frac{1}{\delta} \right)}{\sqrt{N}}.
\end{align*}
Combining the above bounds valid for $k = 0,\ldots,\bar{m}$, and further recalling \eqref{eq:bound_epsk_k>barm} we conclude that 
\begin{equation} \label{eq:mu_explicit}
\eps_k \leq \frac{k}{N} + \max \left\{\frac{1}{\sqrt{N}} , 2\frac{\ln \left( \frac{1}{\delta} \right)}{\sqrt{N}} \right\},
\end{equation}
holds for all $k = 0,\ldots,N$. 

Using \eqref{eq:mu_explicit} in \eqref{eq:epsk_is_ok} yields
$$
\Pr \left\{ \bphi_N \leq \frac{|\co(\bz_1,\ldots,\bz_N)|}{N} + \max \left\{\frac{1}{\sqrt{N}} , 2\frac{\ln \left( \frac{1}{\delta} \right)}{\sqrt{N}} \right\} \right\} \geq 1-\delta,
$$
which provides an explicit expression for $\mu(N,\delta)$ in the statement of Theorem \ref {change of compression rate}. This concludes the proof. \qed

\begin{remark}[On the origin of Theorem \ref{change of compression rate}] \label{rmk:origin}
	
We find it useful to clarify the reasoning that led us to Theorem~\ref{change of compression rate}. Besides motivating the steps in the preceding proof, this discussion provides further insight into the function $\eps_{(\cdot)}$ introduced in \eqref{eq:epsk4as_expression}. 

In Theorem~\ref{thm:main_compression}, it is relatively easy to show that, for a given $\hat{k}$, setting $\eps_{\hat{k}} = \eps$ and $\eps_k = 1$ for all $k \neq \hat{k}$, which corresponds to only certifying $\bphi_N$ when $|\co(\bz_1,\ldots,\bz_N)| = \hat{k}$, yields an optimal value $\gamma^\ast_N$ of \eqref{aux_opt_pb} attained by setting $\lambda_{\hat{m}} \neq 0$ for some suitable $\hat{m} \in \{\hat{k},\ldots,N\}$ and $\lambda_m = 0$ for all $m \neq \hat{m}$. More precisely, this value is $\lambda_{\hat{m}} = \frac{{N \choose \hat{k}}}{{\hat{m} \choose \hat{k}}}(1-\eps)^{N-\hat{m}}$. Conversely, if Theorem \ref{thm:main_compression} is used in reverse direction to determine $\eps_{(\cdot)}$ such that $\bphi_N \leq \eps_{|\co(\bz_1,\ldots,\bz_N)|}$ with a user-specified confidence $1-\delta$, then obtaining the smallest possible value for $\eps_{\hat{k}}$ for a given $\hat{k}$, possibly at the expense of other values $\eps_k$, requires setting $\lambda_{\hat{m}} = \delta$ for some suitable $\hat{m} \in \{\hat{k},\ldots,N\}$, and therefore $\lambda_m = 0$ for all $m \neq \hat{m}$. The interested reader is referred to Appendix \ref{sec:additional_results} for a formal proof of these claims. 

Interestingly, when setting $\lambda_{\hat{m}} = \delta$ as described in the preceding paragraph, Theorem \ref{thm:main_compression} allows one to determine not only $\eps_{\hat{k}}$, but also $\eps_k$ for all $k \neq \hat{k}$. The resulting expressions coincides with that in \eqref{eq:epsk4as_expression} with $\bar{m}$ replaced by $\hat{m}$. As a consequence, one sees that it is not necessary to set $\eps_k = 1$ for $k \neq \hat{k}$ in order to maintain confidence $1-\delta$: there is some room to decrease the other $\eps_k$'s as well. 

However, although the resulting value of $\eps_{\hat{k}}$ is the smallest possible according to the result mentioned above, the same is not true for the other values $\eps_k$ with $k \neq \hat{k}$. In particular, for values $\check{k}$ significantly larger than $\hat{k}$, $\eps_{\check{k}}$ does not come close to its smallest possible value, which can instead be obtained by selecting $\lambda_{\check{m}} = \delta$ for a suitable $\check{m}$. 

In the light of the above, the idea behind the derivation of Theorem~\ref{change of compression rate} arose from the empirical observation that, although $\hat{m}$ and $\check{m}$ can differ significantly for small values of $N$, they come close in relative terms when $N$ increases. Asymptotically in $N$, a single choice $\bar{m} = N - \sqrt{N}$ is sufficient to achieve nearly optimal values of the $\eps_k$'s, simultaneously for all $k$'s, leading to the margin that vanishes at rate $O(1/\sqrt{N})$. \qedstar
\end{remark}

\bibliographystyle{plain} 
\bibliography{RB-convergence-rate}

@article{Berger2025,
	author  = {G.O. Berger},
	title   = {Improved Compression-Based PAC Bounds for General Scenario Decision Making},
	journal = {ArXiv : 2501.08884},		
	year    = {2025},	
	url = {https://arxiv.org/abs/2501.08884},
}

@book{DevGyoLug96,
	author = {L. Devroye , L. Gy\"orfi , G. Lugosi},
	title = {A Probabilistic Theory of Pattern Recognition},
	year = {1996},
	publisher = {Springer},
}

@ARTICLE{CampiGaratti2021,
	author={M.C. Campi and S. Garatti},
	title={A Theory of the Risk for Optimization with Relaxation and its Application to Support Vector Machines},
	journal={Journal of Machine Learning Research},
	year={2021},
	volume={22},
	number={288},
	pages={1-38},
}

@INPROCEEDINGS{CGR2015,  
	author={Campi, M.C. and Garatti, S. and Ramponi, F.A.},  
	booktitle={Proceedings of the 54th Conference on Decision and Control},  
	title={Non-convex scenario optimization with application to system identification},  
	year={2015},
	pages={4023-4028},
}

@article{GarCa2025,
	title={Non-convex scenario optimization},
	author={Garatti, S. and Campi, M.C.},
	journal={Mathematical Programming},
	volume={209},
	pages={557--608},
	year={2025},
	publisher={Springer}
}

@inproceedings{P2L2023,
	author = {Paccagnan, D. and Campi, M.C. and Garatti, S.},
	booktitle = {Advances in Neural Information Processing Systems (NeurIPS 2023)},
	pages = {18165--18185},
	title = {The {P}ick-to-{L}earn algorithm: Empowering compression for tight generalization bounds and improved post-training performance},
	volume = {36},
	year = {2023},
}

@article{P2L2026,
	title={{P}ick-to-{L}earn for {S}ystems and {C}ontrol: data-driven Synthesis with State-of-the-art Safety Guarantees}, 
	author={D. Paccagnan and D. Marks and M.C. Campi and S. Garatti},
	year={2025},
	journal = {ArXiv (article id. arXiv:2512.04781)},
	note={To appear in IEEE Control Systems, vol. 46, no. 6, 2026},
}

@article{Reitzner2003,
	author  = {M. Reitzner},
	title   = {Random polytopes and the {E}fron-{S}tein jackknife inequality},
	journal = {Annals of Probability},
	volume  = {31},
	number  = {4},
	pages   = {2136-2166},
	year    = {2003},	
}

@article{RenyiSulanke63,
	author  = {A. R\'enyi and R. Sulanke},
	title   = {Uber die konvexe H\"ulle von n zuf\"allig gew\"ahlten Punkten},
	journal = {Z. Wahrscheinlichkeitstheorie verw Gebiete},
	volume  = {2},
	number  = {1},
	pages   = {75-84},
	year    = {1963},	
}

@article{Raynaud1970,
	author  = {H. Raynaud},
	title   = {Sur l'enveloppe convexe des nuages de points al\'eatoires dans $\mathbb{R}^n$},
	journal = {Journal of Applied Probability},
	volume  = {7},
	number  = {1},
	pages   = {35-48},
	year    = {1970},
}

@article{Barany1992,
	author  = {I. B{\'a}r{\'a}ny},
	title   = {Random polytopes in smooth convex bodies},
	journal = {Mathematika},
	volume  = {39},
	number  = {1},
	pages   = {81-92},
	year    = {1992},
}

@ARTICLE{Barany2008,
	author={I. B{\'a}r{\'a}ny},
	title={Random points and lattice points in convex bodies},
	journal={Bullettin of the American Mathematical Society},
	year={2008},
	volume={45},
	number={3},
	pages={339-365},
}

@ARTICLE{CampiGaratti2023,
	author={M.C. Campi and S. Garatti},
	title={Compression, Generalization and Learnings},
	journal={Journal of Machine Learning Research},
	year={2023},
	volume={24},
	number={339},
	pages={1-74},
}

@InProceedings{pmlr-v125-bousquet20a,
	title = {Proper Learning, {H}elly Number, and an Optimal {SVM} Bound},
	author = {O. Bousquet and S. Hanneke and S. Moran and N. Zhivotovskiy},
	booktitle = {Proceedings of 33rd Conference on Learning Theory},
	pages = {582--609},
	year = {2020},
	volume = {125},
	series = {Proceedings of Machine Learning Research},
}

@article{Campi2010,
	author = {M.C. Campi},
	title = {Classification with Guaranteed Probability of Error},
	journal = {Machine Learning},
	pages = {63-84},
	volume = {80},
	year = {2010}
}

@ARTICLE{CareRamponiCampi2018,
	author={A. Car\'e and F.A. Ramponi and M.C. Campi.},
	title={A new classification algorithm with guaranteed sensitivity and specificity for medical applications},
	journal={IEEE Control Systems Letters},
	year={2018},
	volume={2},
	number={},
	pages={393-398},
}

@ARTICLE{CortesVapnik1995,
	author={C. Cortes and V. Vapnik},
	title={Support Vector Networks},
	journal={Machine  Learning},
	year={1995},
	volume={20},
	pages={273-297},
}

@ARTICLE{FloydWarmuth1995,
	author={S. Floyd and M. Warmuth},
	title={Sample compression, learnability, and the {V}apnik-{C}hervonenkis Dimension},
	journal={Machine Learning},
	year={1995},
	volume={21},
	pages={269-304},
}

@ARTICLE{GarCa2019,
	author={S. Garatti and M.C. Campi},
	title={Risk and Complexity in Scenario Optimization},
	journal={Mathematical Programming},
	year={2022},
	volume={191},
	number={1},
	pages={243-279},
}

@ARTICLE{GraepelHerbrichShawe-Taylor2005,
	author={T. Graepel and R. Herbrich and J.Shawe-Taylor},
	title={{PAC}-{B}ayesian compression bounds on the prediction error of learning algorithms for classification},
	journal={Machine Learning},
	year={2005},
	volume={59},
	pages={55-76},
}

@article{HanKonSabWei2021,
	author = {S. Hanneke and A. Kontorovich and S. Sabato and R. Weiss},
	title = {{Universal Bayes consistency in metric spaces}},
	journal = {The Annals of Statistics},
	volume = {49},
	number = {4},
	pages = {2129 - 2150},
	year = {2021},
}

@INPROCEEDINGS{HannekeKontorovich_2021,
	author={S. Hanneke and A. Kontorovich},
	title={Stable sample compression schemes: new applications and an optimal {SVM} margin bound},
	booktitle = {Proceedings of the 32nd International Conference on Algorithmic Learning Theory},
	year={2021},
	volume={132},
	pages={1-25},
	series = {Proceedings of Machine Learning Research},
}

@ARTICLE{HelSloWar_1990,
	author={D. Helmbold and R. Sloan and M.K. Warmuth},
	title={Learning nested differences of intersection-closed concept classes},
	journal={Machine Learning},
	year={1990},
	volume={5},
	pages={165-196},
}

@inproceedings{KonSabWei2017,
	author = {A. Kontorovich and S. Sabato and R. Weiss},
	booktitle = {Advances in Neural Information Processing Systems 30 (NeurIPS 2017)},
	title = {Nearest-Neighbor Sample Compression: Efficiency, Consistency, Infinite Dimensions},
	year = {2017},
	address = {Long Beach, CA},
}

@article{KonSabUrn2018,
	author  = {A. Kontorovich and S. Sabato and R. Urner},
	title   = {Active Nearest-Neighbor Learning in Metric Spaces},
	journal = {Journal of Machine Learning Research},
	year    = {2018},
	volume  = {18},
	number  = {195},
	pages   = {1--38},
}

@ARTICLE{LittlestoneWarmuth1986,
	author={N. Littlestone and M. Warmuth},
	title={Relating data compression and learnability},
	journal={Unpublished manuscript},
	year={1986},
	volume={},
	pages={},
}

@book{VapnikChervonenkis1974,
	author = {V. Vapnik and A. Chervonenkis},
	title = {Theory of Pattern Recognition},
	year = {1974},
    publisher = {Nauka, Moscow}
}

@book{SchSmola_BOOK,
	author = {B. Sch\"olkopf  and A.J. Smola},
	title = {Learning with kernels},
	subtitle = {support vector machines, regularization, optimization, and beyond},
	year = {1998},
	publisher = {MIT press},
}

@BOOK{Shalev-Shwartz_Ben-David_2014,
	author =       {S. Shalev-Shwartz and S. Ben-David},
	title =        {Underwstanding Machine Learning (from theory to algorithms)},
	publisher =    {Cambridge University Press},
	year =         {2014},
	address =      {New York, NY},
}

\appendix

\section{Valid margin for empirical means}
\label{appendix empirical mean} 
An application of the Central Limit Theorem (CLT) gives 
\begin{eqnarray}
\prob \left\{ p(A) \leq \frac{1}{N}\sum_{i=1}^N  \mathbb{1} (\bz_i \! \in \! A) + \frac{\bar{\varepsilon}}{\sqrt{N}} \right\} 
& = & 
\prob \left\{  \sqrt{N} \left( \frac{1}{N}\sum_{i=1}^N  \mathbb{1} (\bz_i \! \in \! A) - p(A)\right) \geq -\bar{\varepsilon} \right\} \nonumber \\
& & \xrightarrow[N \to \infty]{} \frac{1}{\sqrt{2 \pi} \sigma_A}\int_{-\bar{\varepsilon}}^{\infty} \ee^{-\frac{x^2}{2\sigma_A^2}} \; \dd x, \nonumber 
\end{eqnarray}
where $\sigma_A^2 = \E \big[ ( \mathbb{1} (\bz_i \! \in \! A) - p(A))^2 \big] = p(A) \cdot (1-p(A))$. Since $p(A) \cdot (1-p(A)) \leq 1/4$, it holds that $\frac{1}{\sqrt{2 \pi} \sigma_A}\int_{-\bar{\varepsilon}}^{\infty} \ee^{-\frac{x^2}{2\sigma_A^2}} \; \dd x \geq \sqrt{\frac{2}{\pi}} \int_{-\bar{\varepsilon}}^{\infty} \ee^{-2 x^2} \; \dd x = 1 - \sqrt{\frac{2}{\pi}} \int_{\bar{\varepsilon}}^{\infty} \ee^{-2 x^2} \; \dd x = 1 - \delta$. Moreover, any margin that is a $o(1/\sqrt{N})$\footnote{For a positive function $g(N)$, notation $o(g(N))$ indicates a positive function $f(N)$ such that $\lim_{N \to \infty} f(N)/g(N) = 0$.} gives: 
\begin{eqnarray}
\prob \left\{ p(A) \leq \frac{1}{N}\sum_{i=1}^N  \mathbb{1} (\bz_i \! \in \! A) + o(1/\sqrt{N}) \right\} =
\prob \left\{  \sqrt{N} \left( \frac{1}{N}\sum_{i=1}^N  \mathbb{1} (\bz_i \! \in \! A) - p(A)\right) \geq -\frac{o(1/\sqrt{N})}{1/\sqrt{N}} \right\}, \nonumber 
\end{eqnarray}
which tends to $1/2$ whenever $p(A)$ is not zero or one. 

\section{Term $\bk \boldsymbol{/} \boldsymbol{N}$ is strictly necessary}
\label{appendix term added}

To show that term $\bk/N$ is strictly necessary, we exhibit an example in which reducing the right-hand side of the inequality inside the probability in \eqref{need of margin} below $\bk/N$ generates an invalid result. 

Suppose that each $\bz_i$ takes value in $[0,1]$ with uniform distribution. For a given $N$, select a second integer $\bar{k} \! \in \! \{0,1,\ldots,N\}$, and consider the following compression $\co$ and learning algorithm $\algo$ defined on multisets $U$ of points in $[0,1]$: 
\begin{itemize}
	\item[(a)] if $|U| < \bar{k}$, then $\co(U) = U$ and $\algo(U)$ is the empty interval; 
	\item[(b)] if $|U| \geq \bar{k}$ and the number of points in $[0,\bar{k}/N]$ is less than $\bar{k}$, then $\co(U)$ is the multiset containing the $\bar{k}$ smallest values of the $z_i$ in $U$ and $\algo(U)$ is the interval $[z_{(\bar{k})},1]$ where $z_{(\bar{k})}$ denotes the $\bar{k}$-th smallest $z_i$ in $U$;
	\item[(c)] if $|U| \geq \bar{k}$ and the number of points in $[0,\bar{k}/N]$ is greater than  or equal to $\bar{k}$, then the compression consists of all points in $[0,\bar{k}/N]$ and $\algo(U)$ is the interval $[\bar{k}/N,1]$.
\end{itemize}
Thus, the hypotheses returned $\algo$ are intervals contained in $[0,1]$ and we say that the interval $\algo(U)$ is inappropriate for $z$ if $z \notin \algo(U)$. Note that $\risk(\algo(U)) \geq \bar{k}/N$ for any $U$, because we always have $\algo(U) \subseteq [\bar{k}/N,1]$.

Next, let us verify that we are in the setting of Theorem \ref{Theorem risk}. The fact that $\co$ is preferent is straightforward; as for the reconstruction function, we can simply take $\algo$ itself, since $\algo(S) = \algo(\co(S))$ by the very definition of $\algo$.

Now, in the present setup, $\bk = |\co(\bz_1,\ldots,\bz_N)|$, because all inappropriate points are already included in the compression. Moreover, $|\co(\bz_1,\ldots,\bz_N)| = \bar{k}$ occurs whenever the number of points in $[0,\bar{k}/N]$ is less than or equal to $\bar{k}$, and it is claimed that this condition is verified with probability at least $0.5$. Indeed, this is the probability with which a Binomial distribution with parameters $N$ and $p = \bar{k}/N$ takes a value less than or equal to $Np$, which is at least $0.5$ because $Np$ is the median of the distribution. Summarizing: with probability at least $0.5$, $\bk$ takes the value $\bar{k}$, and the corresponding risk is no smaller than $\bar{k}/N$ as already observed. This implies that, for any $\delta < 0.5$, the right-hand side of the inequality inside the probability in \eqref{need of margin} cannot be reduced below $\bk/N$ when $\bk$ takes the value $\bar{k}$. Observing now that $\bar{k}$ is an arbitrary value in $\{0,1,\ldots,N\}$, the result is obtained that the term $\bk/N$ cannot be reduced for any value of $\bk$ in its domain of definition.

\section{Proof of the results mentioned in Remark~\ref{rmk:origin}} \label{sec:additional_results}

We begin by showing that, if $\eps_{\hat{k}} = \eps$ for a given $\hat{k}$ and $\eps_k = 1$ for all $k \neq \hat{k}$, then the optimal value $\gamma^\ast_N$ of \eqref{aux_opt_pb} in Theorem~\ref{thm:main_compression} is achieved by setting $\lambda_{\hat{m}} = \frac{{N \choose \hat{k}}}{{\hat{m} \choose \hat{k}}}(1-\eps)^{N-\hat{m}}$ for some suitable $\hat{m} \in \{\hat{k},\ldots,N\}$ and $\lambda_m = 0$ for all $m \neq \hat{m}$.\footnote{In view of Theorem~\ref{thm:main_compression}, this implies that when $\eps_{\hat{k}} = \eps$ and $\eps_k = 1$ for $k \neq \hat{k}$ it holds that 
\begin{equation} \label{eq:min_m_bound}
\Pr \{ \bphi_N \leq \eps_{|\co(\bz_1,\ldots,\bz_N)|} \} \geq 1-\min_{m \in \{\hat{k},\ldots,N\}} \frac{{N \choose \hat{k}}}{{m \choose \hat{k}}}(1-\eps)^{N-m}.
\end{equation}
Interestingly, the recent article \cite{Berger2025} shows that \eqref{eq:min_m_bound} can also be proven with an argument that makes use of elementary mathematics, without resorting to the theoretical apparatus underlying the more general Theorem~\ref{thm:main_compression}. Note that, although \cite{Berger2025} is framed in the context of data-driven decision making as in \cite{GarCa2019,GarCa2025}, the same argument can be adapted to the compression framework considered in the present paper.}

Start by considering the collection $\calC$ of functions ${m \choose \hat{k}}(1-\alpha)^{m-\hat{k}}$ for $m=\hat{k},\ldots,N-1$. As a simple calculation reveals, for $\alpha \in [0,1]$ we have ${m \choose \hat{k}}(1-\alpha)^{m-\hat{k}} \geq {m+1 \choose \hat{k}}(1-\alpha)^{m+1-\hat{k}}$ for $\alpha \geq \frac{\hat{k}}{m+1}$ and ${m \choose \hat{k}}(1-\alpha)^{m-\hat{k}} \leq {m+1 \choose \hat{k}}(1-\alpha)^{m+1-\hat{k}}$ for $\alpha \leq \frac{\hat{k}}{m+1}$. Iterating this comparison shows that
\begin{equation} \label{eq:order_of_func}
{m \choose \hat{k}}(1-\alpha)^{m-\hat{k}} \geq {m' \choose \hat{k}}(1-\alpha)^{m'-\hat{k}}  \quad \text{when} \quad \frac{\hat{k}}{m+1} \leq \alpha \leq \frac{\hat{k}}{m}
\end{equation}
for every $m' \in \{\hat{k},\ldots,N\}$, and ${N \choose \hat{k}}(1-\alpha)^{N-\hat{k}}$ is the largest function in the interval $0 \leq \alpha \leq\frac{k}{N}$. 

With $\eps_{\hat{k}} = \eps$ and $\eps_k = 1$ for $k \neq \hat{k}$, consider a feasible choice $\lambda_0, \lambda_1, \ldots, \lambda_N$ for problem \eqref{aux_opt_pb}. For $k \neq \hat{k}$, \eqref{aux_opt_pb} is satisfied independently of the choices of the $\lambda_m$'s. However, the $\lambda_m$'s must satisfy \eqref{aux_opt_pb} for $k = \hat{k}$, which gives 
\begin{equation} \label{eq:simplified_feas_cond}
{N \choose \hat{k}} (1-\eps)^{N-\hat{k}} \leq  \sum_{m=\hat{k}}^N \lambda_m {m \choose \hat{k}} (1-\eps)^{m-\hat{k}}.
\end{equation}
Let $\hat{m}$ be the integer in $\{\hat{k},\ldots, N\}$ such that $\frac{\hat{k}}{\hat{m}+1} \leq \eps \leq \frac{\hat{k}}{\hat{m}}$,\footnote{While condition $\frac{\hat{k}}{\hat{m}+1} \leq \eps \leq \frac{\hat{k}}{\hat{m}}$ typically defines $\hat{m}$ uniquely, when $\eps$ takes an extreme point of the interval $[\frac{\hat{k}}{\hat{m}+1}, \frac{\hat{k}}{\hat{m}}]$, two consecutive integers $\hat{m}$ might satisfy the condition. In this case, any of them can be chosen. It is also understood that $\hat{m}= N$ if $\eps \leq \frac{\hat{k}}{N}$.} and define $\lambda'_{\hat{m}} = \sum_{m=0}^N \lambda_m$ and $\lambda'_m = 0$ for $m \neq \hat{m}$. In view of \eqref{eq:simplified_feas_cond} and \eqref{eq:order_of_func}, we then obtain 
\begin{align*}
{N \choose \hat{k}} (1-\eps)^{N-\hat{k}}  & \leq  \sum_{m=\hat{k}}^N \lambda_m {m \choose \hat{k}} (1-\eps)^{m-\hat{k}} \\
& \leq \left( \sum_{m=\hat{k}}^N \lambda_m \right) {\hat{m} \choose \hat{k}} (1-\eps)^{\hat{m}-\hat{k}} \\
& = \lambda'_{\hat{m}} {\hat{m} \choose \hat{k}} (1-\eps)^{\hat{m}-\hat{k}},
\end{align*}
which shows that $\lambda'_0, \lambda'_1, \ldots, \lambda'_N$ is feasible for \eqref{aux_opt_pb} at $\alpha = \eps$, and thereby for any $\alpha \in [0,1]$, as it is easy to verify. Moreover, the value corresponding to these $\lambda'_m$'s is the same as that of the $\lambda_m$'s because $\sum_{m=0}^N \lambda_m = \sum_{m=0}^N \lambda'_m = \lambda'_{\hat{m}}$. This entails that the optimal solution to \eqref{aux_opt_pb} con be obtained by focusing on solutions of the form $\lambda'_{\hat{m}} \neq 0$ and $\lambda'_m = 0$ for $m \neq \hat{m}$. Finally, the minimum value of $\lambda'_{\hat{m}}$ is achieved by selecting $\lambda'_{\hat{m}}$ so that ${N \choose \hat{k}} (1-\eps)^{N-\hat{k}} = \lambda'_{\hat{m}} {\hat{m} \choose \hat{k}} (1-\eps)^{\hat{m}-\hat{k}}$, which gives $\lambda'_{\hat{m}} = \frac{{N \choose \hat{k}}}{{\hat{m} \choose \hat{k}}}(1-\eps)^{N-\hat{m}}$.

Consider now the converse use of Theorem~\ref{thm:main_compression}, where, for given $\lambda_m \geq 0$, $m = 0,\ldots,N$, with $\sum_{m=0}^N \lambda_m = \delta$, the $\eps_k$'s are obtained as the values that make the constraints in \eqref{aux_opt_pb} tight. We show that the $\lambda_m$'s that yield, for a given $\hat{k}$, the smallest value for $\eps_{\hat{k}}$ are $\lambda_{\hat{m}} = \delta$, for some $\hat{m} \in \{\hat{k},\ldots,N\}$, and $\lambda_m = 0$ for $m \neq \hat{m}$. 

Indeed, this result follows from what has been established above. Consider any partition of $\delta$ into multiple $\lambda_m$'s so that $\sum_{m=0}^N \lambda_m = \delta$, leading to certain $\eps_k$, $k=0,\ldots,N$. Then, by setting $\eps_k = 1$ for all $k \neq \hat{k}$ while keeping $\eps_{\hat{k}}$ unchanged,  solving \eqref{aux_opt_pb} for optimality must result in a solution of the form $\lambda^\ast_{\hat{m}} \leq \delta$ for some $\hat{m}$ and $\lambda^\ast_m = 0$ for $m \neq \hat{m}$. Now set $\lambda'_{\hat{m}} = \delta$ and $\lambda'_m = 0$ for $m \neq \hat{m}$ and recompute the corresponding $\eps'_k$, $k=0,\ldots,N$. The resulting $\eps'_{\hat{k}}$ is either equal to the original $\eps_{\hat{k}}$ if $\lambda^\ast_{\hat{m}} = \delta$ or smaller if $\lambda^\ast_{\hat{m}} < \delta$ (this is because increasing the $\lambda_m$'s always leads to a decrease of the $\eps_k$'s). This shows that for every choice of the $\lambda_m$'s such that $\sum_{m=0}^N \lambda_m = \delta$, one can always construct  an alternative choice $\lambda'_{\hat{m}} = \delta$ and $\lambda'_m = 0$ for $m \neq \hat{m}$ that yields the same or a smaller value of $\eps_{\hat{k}}$. This concludes the proof.

\end{document}